\documentclass{article}

\usepackage{amsmath, amssymb, amsfonts, amsthm}
\usepackage{algorithm}
\usepackage[title, titletoc]{appendix}
\usepackage[noend]{algpseudocode}

\usepackage{diagbox}
\usepackage{multirow}
\usepackage{graphicx}

\usepackage{hyperref}

\allowdisplaybreaks

\newtheorem{theorem}{Theorem}
\newtheorem{lemma}[theorem]{Lemma}

\newtheorem{proposition}[theorem]{Proposition}

\theoremstyle{definition}

\let\oldproofname=\proofname
\renewcommand{\proofname}{\rm\bf{\oldproofname}}

\newcommand{\indep}{\,\rotatebox[origin=c]{90}{$\models$}\,}

\def\bx{\mathbf x}
\def\bm{\mathbf m}
\def\bmu{\boldsymbol \mu}
\def\bE{\mathbf E}
\def\b1{\mathbf 1}

\def\RR{\mathbb R}

\newcommand\dsum{\displaystyle\sum}

\begin{document}

\title{Overlapping Sliced Inverse Regression for Dimension Reduction}
\author{Ning Zhang \\
{\small Computational PhD Program, Middle Tennessee State University}\\
{\small 1301 E Main Street, Murfreesboro, TN 37132, USA}\\
{\small Email: ningzhang0123@gmail.com}\\
\qquad \\
Zhou Yu \\
{\small School of Statistics, East China Normal University, }\\ 
{\small Shanghai 200241, China } \\
{\small Email: zyu@stat.ecnu.edu.cn}\\
\qquad \\
Qiang Wu\\
{\small Department of Mathematical Sciences and Computational PhD Program,}\\
 {\small Middle Tennessee State University}\\ 
{\small 1301 E Main Street, Murfreesboro, TN 37132, USA}\\
{\small Email: qwu@mtsu.edu}
}
\date{}
\maketitle

\begin{abstract}
Sliced inverse regression (SIR) is a pioneer tool for supervised dimension reduction.
It identifies the effective dimension reduction space,
the subspace of significant factors with intrinsic lower dimensionality.
In this paper, we propose to refine the SIR algorithm through an overlapping slicing scheme.
The new algorithm, called overlapping sliced inverse regression (OSIR),
is able to estimate the effective dimension reduction space
and determine the number of effective factors more accurately.
We show that such overlapping procedure has the potential to identify the information
contained in the derivatives of  the inverse regression curve,
which helps to explain the superiority of OSIR.
We also prove that OSIR algorithm is $\sqrt n $-consistent
and verify its effectiveness by simulations and real applications.

\bigskip


\end{abstract}

\section{Introduction}
\label{sec:intro}

Regression analysis is a common tool to identify the relationship between 
multivariate predictor $\bx = (x_1,x_2,\ldots, x_p)^\top \in \mathbb{R}^p$
and scalar response $y$. When an appropriate and reasonable model is prespecified, 
we can adopt standard parametric modeling techniques,
such as the maximum likelihood estimation 
or the least squares method 
to make statistical inferences.
When no persuasive model is available, 
we can use nonparametic modeling methods, such as local smoothing, 
to derive information from the data.
When $y \in \mathbb{R}$, many smoothing techniques are available. 
Although nonparametric regression is more data adaptive,
its performance deteriorates fast as the predictor dimension grows.
High-dimensional datasets present many mathematical challenges as well as some opportunities, 
and are bound to give rise to new theoretical developments\cite{fodor2002survey}.
To balance the modeling bias in parametric regression and 
``curse of dimensionality" in nonparametric regression for high dimensional data, 
semiparametric model is often a good alternative, which is defined as follows:
\begin{equation}
\label{equ:model1}
y = f( \beta_1^\top \bx, \beta_2^\top \bx, \ldots, \beta_K^\top \bx, \epsilon),
\end{equation}
where $\beta_k \in \mathbb R ^p$ is a $p\times 1$ vector and $\epsilon$ is independent of $\bx$. 
Model \eqref{equ:model1} is equivalent to
\begin{equation}
\label{equ:model2}
y  \indep  \bx|B^\top \bx
\end{equation}
where $\indep$ represents ``statistical independence" and $B=(\beta_1,\ldots,\beta_K)$ is a $p \times K$ matrix.
The column space spanned by $B$ is called the effective dimension reduction (EDR) space, 
which is denoted as $S_{y|\bx}$. To recover the EDR
space $S_{y|x}$ and its intrinsic dimensionality $K$, 
many algorithms have been developed in past decades; see for example  
\cite{Li1991, CookWeisberg1991, Ker-ChauLi1992, xia2002adaptive, setodji2004k, fukumizu2004kernel,
LiYin2008, wu2008kernel, wu2009localized, wu2010learning, cook2014fused, Yu2016distance} 
and the references therein. 
 
One of the earliest and most popular method to recover the EDR space is sliced inverse regression (SIR) \cite{Li1991}. 
It identifies $S_{y|\bx}$ based on the inverse conditional mean $\bE(\bx|y)$.
Due to its ease to implementation and effectiveness, sliced inverse regression
and its variants have been successfully applied in bioinformatics, hyperspectral image analysis, 
physics,  and many other fields of science; see for example \cite{cook1994using, belhumeur1997eigenfaces,
becker2003sliced, elnitski2003distinguishing, gannoun2004sliced, he2003classification, 
antoniadis2003effective, li2004dimension, DaiLieu2006, zhang2016application}.

In sliced inverse regression, the choice of the number of slices or 
the number of observations in each slice is a subtle yet important issue. 
In \cite{HsingCarroll1992} the $\sqrt{n}-$consistency and asymptotic normality were derived 
when each slice contains only 2 observations.
 In \cite{zhu1995asymptotics},  the asymptotic normality was established 
 when the number of observations in each slice is varying from 2 to $\sqrt{n}$.
In \cite{zhu2010dimension} a cumulative slicing estimation procedure was proposed, 
 which uses a weighted average of SIR kernel matrices
from all possible slicing schemes with two slices. 
Combining the advantages of different slicing scheme, 
a fused estimator was proposed in \cite{cook2014fused}, 
which is proven to be more effective than the original single slicing scheme.

Along the development of cumulative slicing and fused estimation,
we in this paper propose to combining the information among adjacent
slices to refine the SIR algorithm.
While implementation of such refinement is easy, the improvement is significant.
The rest of this paper is as follows.
In Section \ref{sec:sir}, we give a brief review of SIR.
In Section \ref{sec:osir}, we introduce the motivation of overlapping SIR (OSIR) along with its algorithm.
Consistency and dimensionality determination strategy are also discussed.
In Section \ref{sec:connect} we discuss the connections and differences between OSIR and
related algorithms. In Section \ref{sec:sim}
we compare OSIR  with SIR and other related algorithms
through comprehensive simulation studies and
evaluate its effectiveness  on a real data application.
We conclude our paper with some discussions and remarks in Section \ref{sec:conclusion}.

\section{Sliced Inverse Regression}
\label{sec:sir}

The linear conditional mean condition is the key assumption for SIR to effectively recover the EDR space
$S_{y|\bx}$. That is, for any $b\in\RR^p$,
\begin{align}\label{eq:lcm}
\bE(b^\top \bx| \beta_1^\top \bx, \ldots, \beta_K^\top \bx) = c_0  + \sum_{k=1}^ K c_k \beta_k ^\top \bx.
\end{align}
The linear condition mean condition holds true if $\bx$ follows an elliptical contour distribution. 
Under~\eqref{eq:lcm}, the centered inverse conditional mean
$\bE(\bx|y) - \bE(\bx)$ is contained in the linear subspace spanned by $\Sigma \beta_k,  k=1,\ldots,K,$
where $\Sigma$ is the covariance matrix of $\bx.$ Therefore, all or part of the EDR directions can be recovered
by the eigenvectors associated to the nonzero eigenvalues of the following generalized eigenvalue decomposition problem:
\begin{equation}\label{eig}
	\Gamma \beta = \lambda\Sigma\beta
\end{equation}
where $\Gamma=\bE\left((\bE(\bx|y)-\bE(\bx))(\bE(\bx|y)-\bE(\bx))^\top\right)$ is the covariance matrix of
inverse regression curve $\bE(\bx|y)).$
This motivates the use of inverse regression, that is, regressing $\bx$ against $y$, instead of regressing $y$ against $\bx.$

In the sample level, the SIR algorithm can be implemented accordingly.
Let ${(\bx_i, y_i)}_{i=1}^n$ be the i.i.d observations. First compute the sample mean and the sample covariance matrix as
\begin{align*}
\bar\bx = &\frac 1 n  \sum_{i=1}^ n \bx_i,\\
\hat{\Sigma} =& \dfrac{1}{n}\sum\limits_{i=1}^n (\bx_i - \bar{\bx})(\bx_i - \bar{\bx})^\top
\end{align*}
Second, order the response values $y_i$ and bin the data into $H$ slices according to $y_i$.
For $h=1, \ldots, H$, let $s_h$ denote the slice $h$ and  $n_h$ be the number of data points in slice $h.$
Compute the sample probability of each slice as $\hat p_h = \frac{n_h}{n}$ and the sample mean of $\bx_i$ in each slice
as $$\hat \bm_h = \frac 1 {n_h} \sum_{y_i\in s_h} \bx_i.$$
Then the matrix $\Gamma$ is estimated by
\begin{equation*}
	\hat{\Gamma}_H = \sum_{h=1}^H  \hat p_h \left(\hat \bm_h - \bar{\bx}\right)\left(\hat \bm_h - \bar{\bx}\right)^\top.
\end{equation*}
Finally, solve the generalized eigenvalue problem
\begin{equation*}
	\hat{\Gamma}_H\hat{\beta} = \lambda\hat{\Sigma}\hat{\beta}.
\end{equation*}
The EDR directions are estimated by the top $K$ eigenvectors $\hat{\beta}_k, k=1,2,\ldots, K.$

In SIR algorithm, the slice number $H$ is an insensitive parameter provided that $H$ is relatively larger than $K.$
One can select $H$ as large as $\frac n 2$ so that
each slice contains only two points and  SIR algorithm still achieves root-$n$ consistency.
It is also found better to make all slices have similar number of data points,
instead of making all slices having similar interval range in $y$.
Therefore, in practice, if $H$ divides $n$, all slices will have equal
number of data points, that is, $n_h=\frac {n}{H}$ for all $h=1, \ldots, H.$

\section{Refinement by slicing overlapping }
\label{sec:osir}

Keep in mind that $\hat \bm _h$ actually provides a sample estimation for  $\bE(\bx|y),$ $y\in s_h,$
the inverse conditional mean at $y$ within slice $h$.
Under \eqref{eq:lcm}, we know that the centered inverse regression curve
$\bE(\bx|y)-\bE(\bx)$ lies in the subspace spanned by $\Sigma\beta_1,\ldots, \Sigma \beta_K.$
As an estimated vector, $\hat \bm _h-\bar \bx$ is expected to be close to this subspace but not exactly lie in it.
To improve the estimation of $\bE(\bx|y)$, a simple and direct approach is to increase the number of points within each slice.
This, however, is equivalent to decrease the number of slices $H$ and is generally not desirable,  
because $H$ must be larger than $K$.  In practice,
a moderate value of $H$ is preferred as a too small $H$ may lead severe degeneracy and lose EDR information.
Therefore, a natural question becomes that, with an appropriately selected and fixed $H$,
can we take more advantage of the data in hand and estimate inverse regression curve more accurately?
This inspires us to allow slicing overlapping which leads to a refined algorithm
for sliced inverse regression. The new estimator is called overlapping sliced inverse regression (OSIR).

We now describe the OSIR algorithm in detail. For each $h=1, \ldots, H-1$,
we combine slice $s_h$ and its adjacent slice $s_{h+1}$ to form a bundle
and compute the mean of predictors in this bundle
$$\hat\bm_{h:(h+1)} = \frac 1 {n_h+n_{h+1}} \sum_{y_i\in s_{h}\bigcup s_{h+1}} \bx_i,$$
which is expected to be closer to the subspace spanned by $\Sigma\beta_1, \ldots, \Sigma \beta_K$
than $\hat \bm_h$ and $\hat\bm _{h+1}.$ As a result, the OSIR algorithm using kernel matrix
estimated from these bundle means is expected to provide more accurate estimation for the EDR directions.
Note that for each $h=2, \ldots, H-1$, the original slice $s_h$ is the overlapping of two bundles 
and is used twice in the computation of the bundle means.
Thus we make a 50\% adjustment for computing the sample probability of the bundles,
that is, we will use
$$\hat p_{h:(h+1)} =\frac 12 (\hat p_h+\hat p_{h+1})= \frac {n_h+n_{h+1}}{2 n} .$$
The first slice $s_1$ and the last slice $s_H$, however, are used only once.
To make all data points to have the same contribution to the algorithm,
we need further adjustment by adding $\hat\bm_1$ with weight $\frac{\hat p_1}{2}$
 and $\hat\bm_H$ with weight $\frac{\hat p_H}{2}$ towards the estimation of $\Gamma.$
 Taking all these into consideration, we obtain
 $$\begin{array}{rl}
 \hat \Gamma_H^{(1)} = & \dsum_{h=1}^{H-1} \hat p_{h:(h+1)} (\hat \bm _{h:(h+1)}-\bar\bx)(\hat \bm _{h:(h+1)}-\bar\bx)^\top  \\[1em]
 & + \dfrac {\hat p_1}{2} (\hat \bm _1 - \bar\bx)(\hat\bm_1-\bar\bx)^\top
 + \dfrac {\hat p_H}{2} (\hat \bm _H - \bar\bx)(\hat\bm_H-\bar\bx)^\top.
 \end{array}
$$
 This algorithm can be interpreted alternatively as follows. We first duplicate the data so that we have $2n$ data points
 which contain two copies of every original data point. Then we bin the data into $H+1$ bundles with the constraint
 that each bundle can only contain one copy of an original data point. Then the first bundle naturally contains
 one copy of slice 1 and one copy of slice 2, the second big slice contains slice 2 and slice 3, and so on.
 This leaves slice 1 and slice $H$ to be treated separately. In this process the data is replicated once
 or equivalently each slice is overlapped once. Therefore, 
 we call this algorithm level-one overlapping sliced inverse regression (OSIR$_1$).

\subsection{Overlapping codes information of difference}
\label{sec:difference}

Firstly we notice that the level-one  overlapping actually codes the first order difference
(or  the first order derivative in the limiting sense)
of the inverse regression curve, which allows us to
interpret the effectiveness of OSIR from an alternative perspective.

\begin{proposition} \label{Grelation}
We have
$$\hat \Gamma^{(1)}_H = \hat \Gamma_H - \frac 1 2 \sum_{h=1}^{H-1}
\frac{\hat p_h \hat p_{h+1}}{\hat p_h + \hat p_{h+1}} \left(\hat\bm_{h+1}-\hat\bm_{h}\right)\left(\hat\bm_{h+1}-\hat\bm_{h}\right)^\top.$$
In particular if $n_1=n_2 = \ldots = n_H = \frac{n}{H},$ we have
$$\hat \Gamma^{(1)}_H = \hat \Gamma_H - \frac 1 {4H}  \sum_{h=1}^{H-1}
 \left(\hat\bm_{h+1}-\hat\bm_{h}\right)\left(\hat\bm_{h+1}-\hat\bm_{h}\right)^\top.$$
\end{proposition}

Proposition \ref{Grelation} tells that $\Gamma^{(1)}_H$ can be obtained by subtracting from
$\hat\Gamma_H$ a weighted covariance matrix of the first order difference of
sample inverse regression curve $\hat\bm_h.$ The proof is given in Appendix \ref{proof:OSIR1}.

Let $p_h$ be the probability and $m_h$ the mean vector of slice $s_h$.
The population version of  the difference between $\Gamma_H$ and
$\Gamma_H^{(1)}$ is
$$ D^{(1)}_H = \frac 1 2 \sum_{h=1}^{H-1}
\frac{p_h p_{h+1}}{ p_h + p_{h+1}} \left(\bm_{h+1}-\bm_{h}\right)\left(\bm_{h+1}-\bm_{h}\right)^\top.$$
If the inverse regression curve is smooth, then $\bm_{h+1}-\bm_h$ is of order $O(\frac 1 H)$ for large $H$
and codes the information of the first order derivative of $\bE(\bx|y).$
This indicates that $D_H^{(1)}$, the difference between $\Gamma_H$ and $\Gamma_H^{(1)},$
is $O(\frac 1 {H^2})$.
Thus,  if we let $H$ tend to infinity, both OSIR and SIR estimate the covariance matrix $\Gamma$
of the inverse regression curve. But for small or moderate $H$, their difference could be substantive. 

Now let us see why OSIR$_1$  is generally superior to SIR.
We decompose $\hat\bm_{h+1}-\hat\bm_{h}$ as
$\hat{\mathbf v}_h + \hat{\mathbf v}_h^\perp$ where  $\mathbf v_h$ is the component in the subspace
$\Sigma B$ and $\mathbf v_h^\perp$ is the orthogonal component.
Let $\hat V$ and $\hat V^\perp$ be the weighted sample covariance matrices of $\hat{\mathbf v}_h$ and $\hat{\mathbf v}_h^\perp$, respectively.
Then $\hat D_H^{(1)} = \hat V+ \hat V^\perp$ and moreover, we expect
$\hat V\to \mathbf 0$ and $\hat V^\perp \to D_H^{(1)}$ as $n$ becomes large.
Note $\hat{\mathbf v}_h$ contains information of the EDR space, so subtracting $\hat V$ from $\hat\Gamma_H$ reduces effective information.
The orthogonal component $\mathbf v_h^\perp$ measures the deviation of $\hat\bm_h$ from the subspace $\Sigma B$.
Subtracting $\hat V^\perp$ reduces noise and improves EDR space estimation.
We claim that, in general, the impact of reducing noise by subtracting $\hat V^\perp$ is greater than
the loss of effective information resulted from subtracting $\hat V.$
First,  $\hat V$ is of order $O(\frac 1 {H^2})$ for large $n$ when $\bm(y)$ is smooth.
Thus, its impact is minimal even with a moderate $H$.
Second, roughly speaking, the estimation accuracy of SIR algorithms is positively correlated to
signal to noise ratio $\rho=\frac{\sum_{k=1}^K \hat \lambda_k}{\sum_{k=K+1} ^d \hat\lambda_k}.$
In the perfect situation $\hat\lambda_k = 0 $ for $k=K+1, \ldots, d$, the signal to noise ratio is infinity and
the EDR space can be exactly estimated. Let $\gamma_0$ measure the effective information contained in $\hat V$
and $\gamma_1$ the noise level in $\hat V^\perp.$ Then the signal to noise ratio of OSIR$_1$ becomes
$\rho^{(1)}=\frac{\sum_{k=1}^K \hat \lambda_k -\gamma_0}{\sum_{k=K+1} ^d \hat\lambda_k-\gamma_1}.$  It is larger
than $\rho$ provided that
\begin{equation} \label{improvecond}
\gamma_1 > \gamma_0\frac{\sum_{k=K+1} ^d \hat\lambda_k}{\sum_{k=1}^K \hat \lambda_k}.
\end{equation}
In most solvable problems $\sum_{k=K+1} ^d \hat\lambda_k$ should be much smaller than
${\sum_{k=1}^K \hat \lambda_k}$ (for otherwise no algorithm works due to very small signal to noise ratio).
Thus \eqref{improvecond} can be easily fulfilled so that OSIR$_1$ outperforms SIR.

\subsection{The $\sqrt{n}$ consistency}
\label{sec:consistency}

For supervised dimension reduction methods such as SIR, 
the $\sqrt n$-consistency and asymptotic normality not only provides theoretical guarantee
for the asymptotic estimation accuracy of the EDR space, but also
establishes the basis of  various strategies  for dimensionality determination.
In this subsection, we show that, for OSIR, 
the $\sqrt n$-consistency and asymptotic normality can be established as follows.

\begin{theorem}\label{error}
Let $(\lambda_k, \beta_k), k=1, \ldots, K$ be the eigenvalue and eigenvectors of the generalized eigendecomposition problem
$$\Gamma_H^{(1)} \beta = \lambda \Sigma \beta$$
and
$(\hat\lambda_k, \hat \beta_k), k=1, \ldots, K$ be the eigenvalue and eigenvectors of the generalized eigendecomposition problem
$$\hat \Gamma_H^{(1)}\beta = \lambda  \hat \Sigma \beta.$$
Assume $\lambda_k, k=1, \ldots, K$ are distinct.
Then there exist a real valued functions $\xi_k(\bx, y)$ and vector values function $\Upsilon_k(\bx, y)$ such that
$$\hat \lambda_k = \lambda_k + \frac 1 n \sum_{i=1}^n \xi_k (\bx_i, y_i) + o_P(\tfrac 1 {\sqrt n})$$
and
$$\hat\beta_k = \beta_k + \frac 1 n \sum_{i=1}^ n  \Upsilon_k(\bx_i, y_i) + o_P(\tfrac 1 {\sqrt n})$$
\end{theorem}

The proof of Theorem \ref{error} is given in Appendix \ref{proof:consistency}.

\subsection{High level overlapping}
\label{sec:high}

The idea of extending OSIR to high level overlapping is natural.
The only tricky point is on the adjustment for the slices at the two ends.
We now illustrate the idea with the level two overlapping.

For level two overlapping we construct bundles using three adjacent base slices.
So for $h=1, \ldots, H-2$, the $h$-th bundle contains data points from base slices $s_h, \ s_{h+1}$ and $s_{h+2}$.
The corresponding bundle probability is computed as
$$\hat p_{h:(h+2)} = \frac 13(\hat p_h + \hat p_{h+1} + \hat p_{h+2})$$
because each base slice is used three times.
The corresponding bundle mean is
$$\hat \bm_{h:(h+2)} 
=\frac{\hat p_h \hat \bm_h + \hat p _{h+1} \hat\bm_{h+1} +\hat p_{h+2} \hat\bm_{h+2}}
{\hat p_h + \hat p_{h+1} + \hat p_{h+2}}. $$
Then we see the slice $s_1$ and $s_H$ are used only once, the slice $s_2$ and $s_{H-1}$ are used twice.
To make all data points have equal contribution in the algorithm, we will not  add them separately.
Instead, we do the adjustment as follows. We combine $s_1$ and $s_2$ as one intermediate bundle, compute  its probability
as $\frac 13(\hat p_1 + \hat p_2)$ and the bundle mean. Then we add slice $s_1$ with probability $\frac 1 3 \hat p_1.$
The last two slices $s_{H-1}$ and $s_H$ are treated analogously. This leads to
$$\begin{array}{rcl}
\hat\Gamma_H^{(2)} & = & \dsum_{h=1}^{H-2} \hat p_{h:(h+2)} (\hat \bm_{h:(h+2)}-\bar\bx)(\hat \bm_{h:(h+2)}-\bar\bx)^\top  \\[1em]
& & + \frac 13 ( \hat p_1 + \hat p_2)  (\hat \bm_{1:2}-\bar\bx)(\hat \bm_{1:2}-\bar\bx)^\top \\[1em]
& & + \frac 13( \hat p_{H-1}+\hat p_H) (\bm_{(H-1):H}-\bar\bx)(\hat \bm_{(H-1):H}-\bar\bx)^\top \\[1em]
 & & + \frac 1 3 \hat p_1 (\hat\bm _1 -\bar\bx) (\hat\bm _1 -\bar\bx)^\top +
  \frac 1 3 \hat p_H (\hat\bm _H -\bar\bx) (\hat\bm _H -\bar\bx)^\top .
\end{array}$$
Again, we can interpret the process as that we first duplicate the
 data twice to obtain three copies of all original data points and
 then  bin the data into $H + 2$ bundles with the constraint that
 each slice can only contain one copy of an original data point.

 We can further extend the algorithm to any overlapping level $L\le H-1.$
 The representation of the associated matrix $\hat\Gamma_H^{(L)}$ will be more complicated
 by using normal notations. But interestingly we can
 have  a unified representation for all $1\le L\le H-1$
 by introducing some ghost slices. To this end, we define null slices for indices
 $h=\ldots, -2, -1, 0 $ and $h=H+1, H+2, H+3,\ldots$ to be slices with
 probability $\hat p_h=0$  and slice mean $\hat\bm _h =\mathbf 0.$
 For each $h$ define
 $$\hat p_{h:h+L} = \frac 1 {L+1} (\hat p_h + \ldots + p_{h+L})$$
 and
 $$\hat\bm_{h:(h+L)} = \frac { \hat p_h \hat\bm_h + \ldots + \hat p_{h+L} \hat \bm_{h+L}}
 { \hat p_h  + \ldots  + \hat p_{h+L} }.$$
 Then for all $1\le L\le H-1,$ we have
 $$\hat\Gamma_H^{(L)} = \sum_{h=-L+1}^{H}  \hat p_{h:h+L} \left( \hat\bm_{h:(h+L)}-\bar\bx\right)
  \left( \hat\bm_{h:(h+L)}-\bar\bx\right)^\top.$$
  The algorithm using $\hat\Gamma_H^{(L)}$ for dimension reduction 
  will be called level-$L$ overlapping sliced inverse regression, or OSIR$_L$.

We notice that the level-two overlapping codes both the first and the second order derivatives of the inverse regression curve.
\begin{proposition} \label{Grelation2}
We have
$$\begin{array}{rcl}
\hat\Gamma_H^{(2)} & = & \hat \Gamma_H -  \dfrac 1 3 \dsum_{h=-1}^{H}
\Big( \dfrac{\hat p_h \hat p_{h+1} +2\hat p_h \hat p_{h+2}}{\hat p_h + \hat p_{h+1} +
 \hat p_{h+2}} \left(\hat\bm_{h+1}-\hat\bm_{h}\right)\left(\hat\bm_{h+1}-\hat\bm_{h}\right)^\top \\[2em]
&  &\quad +   \dfrac{\hat p_{h+1} \hat p_{h+2} +2\hat p_h \hat p_{h+2}}{\hat p_h + \hat p_{h+1}
+ \hat p_{h+2}} \left(\hat\bm_{h+2}-\hat\bm_{h+1}\right)\left(\hat\bm_{h+2}-\hat\bm_{h+1}\right)^\top
 \Big)  \\[1em]
 &&\quad + \dfrac 1 3 \dsum_{h=-1}^{H}
\dfrac{\hat p_h \hat p_{h+2}}{\hat p_h + \hat p_{h+1} + \hat p_{h+2}}
\left(\hat\bm_{h+2}-2\hat\bm_{h+1}+\hat\bm_{h}\right)\left(\hat\bm_{h+2}-2\hat\bm_{h+1}+\hat\bm_{h}\right)^\top  .
\end{array}$$
In particular if $n_1=n_2 = \ldots = n_H = \frac{n}{H},$ we have
$$\begin{array}{rcl}
\hat\Gamma_H^{(2)} & = & \hat \Gamma_H -  \dfrac 2 {3H}  \dsum_{h=1}^{H-1}
 \left(\hat\bm_{h+1}-\hat\bm_{h}\right)\left(\hat\bm_{h+1}-\hat\bm_{h}\right)^\top \\[1em]
& & \quad +   \dfrac 1 {9H}  \dsum_{h=1}^{H-2}
 \left(\hat\bm_{h+2}-2\hat\bm_{h+1}+\hat\bm_{h}\right)\left(\hat\bm_{h+2}-2\hat\bm_{h+1}+\hat\bm_{h}\right)^\top  \\[1em]
 & &\quad + \dfrac{1}{2H}\left(\hat\bm_{2} - \hat\bm_{1}\right)\left(\hat\bm_{2} - \hat\bm_{1}\right)^\top
 + \dfrac{1}{2H}\left(\hat\bm_{H} - \hat\bm_{H-1}\right)\left(\hat\bm_{H} - \hat\bm_{H-1}\right)^\top .
\end{array}$$
\end{proposition}
Proposition \ref{Grelation2} tells that $\Gamma^{(2)}_H$ can be obtained by subtracting from
$\hat\Gamma_H$ a weighted covariance matrix of the first order difference of the sample inverse regression curve 
 and adding a weighted covariance matrix of the second order difference of the
sample inverse regression curve $\hat\bm_h.$ The proof is given in Appendix \ref{proof:OSIR2}.

Similar to  OSIR$_1$ and OSIR$_2$, one can show that OSIR$_L$ codes the information of up to $L$-th order derivatives
of the inverse regression curve. Also, OSIR$_L$ is $\sqrt n$-consistent. The proofs are similar to those for OSIR$_1$ and OSIR$_2$
but the computation and representation of the results are much more complicated. We omit the details.

\subsection{Determine the dimensionality}
\label{sec:dimension}

In practice, the true dimensionality $K$ is unknown and has to be estimated from the data.
For SIR and related algorithms, classical methods for dimensionality determination are the sequential $\chi^2$ test 
based on the asymptotic normality.
This method can also be applied to OSIR. However, as mentioned  in \cite{zhu2010dimension}, it is usually very challenging
because the asymptotic variance has very complicated structure
and the the degree of freedom is difficult to determine.
In this paper, we follow the idea in \cite{zhu2006sliced} and \cite{zhu2010dimension} 
and propose a modified BIC method to determine $K$.
For each $1\le L\le H-1$, let $\hat \lambda_i^{(L)}$ be the eigenvalues
of the generalized eigendecomposition problem $\hat\Gamma_H^{(L)} \beta = \lambda \hat\Sigma\beta$
and assume they are arranged in decreasing order.
Define
$$G^{(L)} (k) = \left. { n \dsum_{i=1}^k \left(\hat \lambda_i^{(L)}\right)^2} \middle/ {\dsum_{i=1}^d \left(\hat \lambda_i^{(L)}\right)^2}
\right.
-\frac{C_n k(k+1)}2$$
and we estimate $K$ by
$$\hat K^{(L)}  = \arg\max_{1\le k\le d} G^{(L)}(k).$$
Since OSIR algorithms are $\sqrt n$-consistent, this criterion is consistent
if $C_n\to\infty$ and $C_n/n\to 0$ as  $n\to \infty.$
A challenging issue remaining is the choice of $C_n$ in a data-driven manner.
We are motivated by \cite{zhu2010dimension} to choose $C_n\sim \frac {n^3/4}{d}.$
At the same time we observe from empirical simulations that smaller penalty should be used for larger $H$.
These motivate us to choose
$C_n = \frac{2 n^{3/4}}{p(L+1)H^{1/2}}.$
It is found to work satisfactory  in many situations,
although universally optimal or problem dependent choices deserve further investigation.

\section{Connections with existing methods}
\label{sec:connect}

From its motivation we see OSIR is so closely related to SIR that it seems needless to say anything
regarding their relationship.
However, it would be interesting to notice that overlapping technique does make OSIR essentially different from SIR
in some situations. 
First, it is pointed out \cite{HsingCarroll1992} that SIR works even when there are only
two observations in one slice. But surely SIR does not work with only one observation in a slice ---
$\hat\Gamma_H$ degenerates to be the same as $\hat\Sigma$ in this case.
OSIR, however, still works even if there is only one point in a slice. Second, 
SIR can be applied to classification problem where each class
 naturally defines a slice.
The design of OSIR algorithm depends on the concept of ``adjacent" slices. 
This prevents its use in classification problems because there is no natural way to define 
two or more classes are ``adjacent"  unless the classification problem is an ordinal one.

Another method that is closely related to OSIR is the
cumulative slicing estimate (CUME) propose in \cite{zhu2010dimension}.
CUME aims to recover the EDR space by
the covariance matrix of the cumulative inverse regression curve
$M(\tilde y) = \bE[\bx| y\le \tilde y].$ Empirically, let
$M(y_i) = \frac{1}{|\{j: y_j\le y_i\}|} \sum_{j: y_j\le y_i} \bx_j$
and
$$\hat\Xi = \frac 1 n \sum_{i=1}^n \left(M(y_i)-\bar \bx\right) \left(M(y_i)-\bar \bx\right) ^\top. $$
CUME estimates the EDR space by solving the generalized eigendecompostion problem
$$\hat\Xi\beta = \lambda \hat\Sigma \beta.$$
It is interesting to notice that, if OSIR has each slice containing only on observation
(so that there are $H=n$ slices) and selects overlapping level $L = n-1$,
then $\hat\Gamma _{n}^{(n-1)} = 2 \hat\Xi.$ Therefore, CUME can be regarded as
special case of OSIR.

\section{Simulations}
\label{sec:sim}

In this section we will verify the effectiveness of OSIR with simulations on artificial data and real applications.
Comparisons will be made with two closely related methods, SIR and CUME.

\subsection{Artificial data}

In the simulations with artificial data, since we know the true model,
we measure the performance by the accuracy of the estimated EDR space and the ability of dimension determination.
For the accuracy of the estimated edr space, we adopt the trace correlation $r(K) =$ trace$(\mathbf{P_B} \mathbf{P_{\hat B}})/K$
used in \cite{ferre1998determining} as the measure, where $\mathbf{P_B}$ and $\mathbf{P_{\hat B}}$ are the projection operators onto
the true edr space $\mathbf B$ and the estimated edr space $\mathbf{\hat{B}},$ respectively.
For the ability of dimension determination, we use the modified BIC type criterion which is suitable for all three methods.
For SIR and OSIR we use the choice for $C_n$ as suggested in Section \ref{sec:dimension} (where note SIR corresponds to $L=0$)
while for CUME we use $C_n = 2n^{3/4}/p$ as suggested in  \cite{zhu2010dimension}.

We performed simulation studies with four different models, three from \cite{Li1991} and one from \cite{zhu2010dimension}.
\begin{eqnarray}
\label{equ:model61}
y & = & x_1 + x_2 + x_3 + x_4 + 0x_5 + \epsilon, \\
\label{equ:model62}
y & = & \exp(x_1 + 2 \epsilon) \\
\label{equ:model63}
y & = & x_1( x_1 + x_2 + 1 ) + \epsilon,\\
\label{equ:model64}
y  &= & \dfrac{x_1}{0.5 + (x_2 + 1.5)^2} + \epsilon,
\end{eqnarray}
where $\mathbf{x} = \lbrack x_1,x_2,\ldots,x_p\rbrack^\top$ follow multivariate normal distribution,
$\epsilon$ follows standard normal distribution, $\mathbf{x}$ and $\epsilon$ are independent.
The experiment setting is as follows.

Model \eqref{equ:model61}:
$n=100$, $p=5$, $K=1$, $\mathbf{\beta} =( 0.5,0.5, 0.5, 0.5, 0 )^\top ;$

Model \eqref{equ:model62}:
$n=100$, $p=5$, $K=1$, $\mathbf{\beta} =( 1,0, 0, 0, 0 )^\top $;

Model \eqref{equ:model63}:   $n=400$, $p=10$, $K=2$, $\mathbf{\beta_1} =( 1,0,0,\ldots,0 )^\top $, $\mathbf{\beta_2} =( 0,1,0\ldots,0 )^\top $;

Model \eqref{equ:model64}: $n=400$, $p=10$, $K=2$, $\mathbf{\beta_1} =( 1,0,0,\ldots,0 )^\top $, $\mathbf{\beta_2} =( 0,1,0\ldots,0 )^\top $.

\noindent We tested $H=5$ and $H=10$. All experiments are replicated 1000 times. The average accuracy of edr estimation in terms of
$r(K)$ values as well as the standard deviations are reported  Table \ref{tab:edrall}. The results indicate for both choices of $H$,
OSIR outperforms SIR and when $H$ and $L$ are corrected selected. OSIR also outperforms CUME.
We notice that both SIR and OSIR show not sensitive to the choice of $H$ provided that it
is sufficiently large relative to the true dimension $K$.
For model \eqref{equ:model61} and \eqref{equ:model62}, since $K=1$,
a choice of $H=5$ already large enough, so we see the result for $H=5$ and $H=10$ are quite similar.
For model \eqref{equ:model63} and \eqref{equ:model64}, since $K=2$,
$H=5$ seems not relatively large enough and the results are slightly worse. When $H$ is increased to $10$
both SIR and OSIR performs better. But the performance improvement is ignorable if we further increase $H$
(results not shown). As for the impact of $L$, we see that the most significant improvement is
from  SIR to OSIR, that is, from $L=0$  to $L=1$. When $L$ further increases,
the performance of OSIR may still improves slightly within a small range, but soon
becomes stable. It seems increasing $L$ does not significantly degrade the performance of OSIR.
Therefore, we assume $L=2$ or $3$ should be good enough for most applications
but, if computational complexity is not a concern,
the user may feel free to choose a large $L$.

\begin{table}
\begin{center}
  \begin{tabular}{  c | c|c | c | c | c  }
    \hline
    \hline
    \multicolumn{2}{c|}{\diagbox{Algorithm}{Model} } & \eqref{equ:model61} & \eqref{equ:model62}  & \eqref{equ:model63} &\eqref{equ:model64}  \\
    \hline
     \multirow{5}{*} {$H=5$} & SIR & 0.9822(0.0013)&0.8658(0.0103)&0.7188(0.0115)&0.6968(0.0117) \\ 
    & OSIR$_1 $ &0.9821(0.0013)&0.8734(0.0094)&0.7419(0.0099)&0.7261(0.0101)\\ 
    & OSIR$_2$  &0.9821(0.0013)&0.8724(0.0094)&0.7489(0.0096)&0.7355(0.0097) \\ 
    & OSIR$_3$  &0.9827(0.0013)&0.8730(0.0094)&0.7471(0.0098)&0.7327(0.0099) \\ 
    & OSIR$_4$  &0.9827(0.00123)&0.8730(0.0094)&0.7471(0.0098)&0.7327(0.0099) \\ \hline
     \multirow{10}{*} {$H=10$} &
        SIR  & 0.9855(0.0011)&0.8689(0.0113)&0.7296(0.0122)&0.7288(0.1230) \\ 
  &  OSIR$_1$  &0.9862(0.0010)&0.8916(0.0082)&0.7709(0.0101)&0.7658(0.0103) \\ 
  &  OSIR$_2$  &0.9859(0.0010)&0.8921(0.0081)&0.7775(0.0094)&0.7726(0.0095) \\ 
  &  OSIR$_3$  &0.9855(0.0011)&0.8902(0.0083)&0.7813(0.0090)&0.7762(0.0090) \\ 
  &  OSIR$_4$  &0.9853(0.0011)&0.8888(0.0084)&0.7855(0.0086)&0.7813(0.0087) \\ 
  &  OSIR$_5$  &0.9854(0.0011)&0.8879(0.0084)&0.7894(0.0086)&0.7862(0.0085) \\ 
  &  OSIR$_6$  &0.9856(0.0011)&0.8878(0.0085)&0.7920(0.0085)&0.7900(0.0084) \\ 
  &  OSIR$_7$  &0.9859(0.0010)&0.8881(0.0085)&0.7924(0.0085)&0.7903(0.0085) \\ 
  &  OSIR$_8$  &0.9861(0.0010)&0.8885(0.0084)&0.7908(0.0086)&0.7879(0.0086) \\ 
  &  OSIR$_9$  &0.9861(0.0010)&0.8885(0.0084)&0.7908(0.0086)&0.7879(0.0086) \\ \hline
  \multicolumn{2}{c|}{CUME}  & 0.9844(0.0012)&0.8781(0.0091)&0.7802(0.0088)&0.7760(0.0089) \\ \hline
  \end{tabular}
  \caption{Accuracy of EDR space estimation by SIR, OSIR and CUME for models in \eqref{equ:model61}-\eqref{equ:model64}.
  \label{tab:edrall}}
\end{center}
\end{table}

Next let us fix $H=10.$ The correctness of dimension determination based on the modified BIC criterion
is summarized in Table \ref{tab:K12} and Table \ref{tab:K34}.  CUME seems underestimate the dimensionality.
It works perfectly for models \eqref{equ:model61} and \eqref{equ:model62} and fails for
models \eqref{equ:model63} and \eqref{equ:model64}. OSIR tends to overestimate the dimensionality
with small $L$ while underestimate the dimensionality with large $L.$
Considering the accuracy  of both EDR subspace estimation and dimensionality determination,
a balanced choice of $L$ is recommended
to be  $L=\lfloor H/2 \rfloor,$ the integer part of $H/2.$

\begin{table}[ht]
\begin{center}
  \begin{tabular}{c|ccc|ccc}
    \hline
    \hline
    \multirow{2}{*}{\diagbox{Algorithm}{Model}}
    & \multicolumn{3}{|c|}{\eqref{equ:model61}}&\multicolumn{3}{|c}{\eqref{equ:model62}}  \\
    \cline{2-7}
    & $\hat{K}<1$ & $\hat{K}=1$ & $\hat{K}>1$ & $\hat{K}<1$ & $\hat{K}=1$ &$\hat{K}>1$ \\
    \hline
    SIR&0&0.698&0.320&0&0.056&0.944\\
    OSIR$_1$ &0&0.896&0.104&0&0.203&0.797 \\
    OSIR$_2$ &0&0.938&0.062&0&0.337&0.663 \\
    OSIR$_3$ &0&0.958&0.042&0&0.422&0.578 \\
    OSIR$_4$ &0&0.972&0.028&0&0.521&0.479 \\
    OSIR$_5$ &0&0.986&0.014&0&0.574&0.426 \\
    OSIR$_6$ &0&0.993&0.007&0&0.618&0.382 \\
    OSIR$_7$ &0&0.994&0.006&0&0.629&0.371 \\
    OSIR$_8$ &0&0.994&0.006&0&0.611&0.389 \\
    OSIR$_9$ &0&0.991&0.009&0&0.568&0.432 \\ \hline
    CUME &0&1&0&0&1&0 \\ \hline
\end{tabular}
\caption{Accuracy of dimensionality determination by SIR, OSIR and CUME for models in \eqref{equ:model61} and \eqref{equ:model62}.
\label{tab:K12}}
\end{center}
\end{table}

\begin{table}[ht]
\begin{center}
  \begin{tabular}{c|ccc|ccc}
    \hline
    \hline
    \multirow{2}{*}{\diagbox{Algorithm}{Model}}
    & \multicolumn{3}{|c|}{\eqref{equ:model63}}&\multicolumn{3}{|c}{\eqref{equ:model64}}  \\
    \cline{2-7}
    & $\hat{K}<2$ & $\hat{K}=2$ & $\hat{K}>2$ & $\hat{K}<2$ & $\hat{K}=2$ &$\hat{K}>2$ \\
    \hline
    SIR&0&0.194&0.806&0&0.189&0.811\\
    OSIR$_1$ &0&0.473&0.527&0&0.513&0.487 \\
    OSIR$_2$ &0&0.702&0.298&0&0.772&0.228 \\
    OSIR$_3$ &0&0.886&0.114&0.002&0.923&0.075 \\
    OSIR$_4$ &0&0.956&0.044&0.004&0.976&0.020 \\
    OSIR$_5$ &0&0.975&0.025&0.008&0.984&0.008 \\
    OSIR$_6$ &0.001&0.982&0.017&0.012&0.986&0.002  \\
    OSIR$_7$ &0.002&0.977&0.021&0.016&0.981&0.003 \\
    OSIR$_8$ &0.002&0.965&0.033&0.018&0.975&0.007 \\
    OSIR$_9$ &0&0.958&0.042&0.013&0.973&0.014  \\ \hline
    CUME &0.999&0.001&0&1&0&0 \\ \hline
\end{tabular}
\caption{Accuracy of dimensionality determination by SIR, OSIR and CUME for models in \eqref{equ:model63}  and \eqref{equ:model64}.
\label{tab:K34}}
\end{center}
\end{table}

\subsection{Real data application}
\label{sec:app}

We test the use of OSIR on the Boston housing price data,
collected by Harrison and Rubinfeld \cite{harrison1978hedonic} for the purpose of
discovering whether or not clean air influenced the value of houses in Boston.
The data consist of 506 observations and 14 attributes.

We first preprocess the data by transforming the attributes according to their distribution shapes.
The logarithm transformation is applied to the response variable
and 4 predictors named as ``crim", ``zn", ``nox", and ``dis".
Square transformation is applied to the predictor ``ptratio".
All other predictors are kept untransformed.

To test the impact of dimensional reduction by SIR and OSIR on the predictive modeling,
we split the data into a training set of 200 observations and a test set of 306 observations,
applied SIR and OSIR on the training set to implement the dimension reduction,
then k-nearest neighbor (kNN) regression is applied to predict the response on the test set.
In the experiment, we choose $H=20.$  We repeat the experiment 100 times.
The dimensionality estimated by modified BIC varies between 2 and 4 due to randomness of the training set.
To avoid loss information and for fair comparison, we fixed $K=4$
instead of estimating it using the modified BIC criterion in this experiment.
The mean squared prediction error and standard deviation is reported in Table \ref{tab:house}.
For comparison purpose we also reported the errors by multiple linear regression (MLR) and kNN regression
before dimension reduction. The results implies that both SIR and OSIR is effective to find the relevant directions
for prediction and OSIR outperforms SIR.

We next investigate at the correlations between the estimated edr directions and the response variable,
 which have also shown in Table \ref{tab:house}. Clearly the first edr directions estimated by OSIR
 has higher correlations than SIR, indicating its better ability to accurately estimate the relevant predictive direction.
To compare the accuracy of the whole edr space estimation, it is reasonable to consider
the weighted average of the correlations of all edr directions, with the weights 
being their corresponding eigenvalues, because eigenvalues measure the importance
of the corresponding edr directions. The results in Table \ref{tab:house} show that OSIR finds
edr space more accurate than SIR. 
OSIR achieves optimal results with $L$ around $\frac {H}{2}=10$
in terms of both predictive accuracy and weight average correlations.

\begin{table}[ht]
\begin{center}
  \begin{tabular}{lcccccc}
    \hline
    \hline
    & & \multicolumn{5}{c}{Correlation to Response} \\ \cline{3-7}
    \multicolumn{1}{c}{Algorithm} & MSE & $\hat\beta_1$ &  $\hat\beta_2$  & $\hat\beta_3$ & $\hat\beta_4$ & Weighted Average \\
    \hline
    SIR&  21.66(0.28) & 0.8290&    0.1560&    0.0929&    0.0940&    0.2933\\
    OSIR$_1$ &19.97(0.28) & 0.8344&    0.1558&    0.0996&    0.0881&    0.3108\\
    OSIR$_2$ &19.84(0.28) &0.8358&    0.1490&    0.0984&    0.0877&    0.3207\\
    OSIR$_3$ &19.83(0.27) &0.8366&    0.1437&    0.1017&    0.0917&    0.3290\\
    OSIR$_5$ &19.71(0.26) &0.8373&    0.1346&    0.1118&    0.0952&    0.3419\\
    OSIR$_{10}$ &19.52(0.25) &0.8387&   0.1224&    0.1144&    0.1030&    0.3564\\
    OSIR$_{15}$ &19.40(0.24) &0.8413&   0.1185&    0.1058&    0.1004&    0.3363\\
    OSIR$_{19}$ &19.42(0.25) &0.8418&   0.1179&    0.1034&    0.1008&    0.3052\\ \hline
    MLR & 21.21(0.30) & & & & & \\
    kNN & 53.53(0.59) & & & & & \\ \hline
\end{tabular}
\caption{Experiment results for Boston housing price data.
\label{tab:house}}
\end{center}
\end{table}

\section{Conclusions and Discussions}
\label{sec:conclusion}

We developed an adjacent slice overlapping technique and applied it to the sliced
inverse regression method. This leads to a new dimension reduction approach
called overlapping sliced inverse regression (OSIR). This new approach is showed to
improve the dimension reduction accuracy by coding the higher order difference (or derivative) information
of the inverse regression curve.
The root-$n$ consistency provides theoretical guarantee for its application.

In this paper we have adopted a modified BIC criterion for
the dimensionality determination for OSIR method.
Several alternative strategies have been proposed
for dimensionality determination for SIR method such
as the $\chi^2$ test \cite{Li1991, bura2001extending, bai2004chi}
and bootstrapping \cite{Barriosa2007}.
We expect these strategies also apply to OSIR 
and would leave it a future research topic for an optimal strategy.

Finally we remark that the purpose of OSIR is to improve the dimension reduction accuracy
in the situation SIR works but does not give optimal estimation. It does not overcome the
degeneracy problem of SIR. Instead, it inherited this problem from SIR. In fact,
all inverse regression based method including SIR, OSIR and CUME face this problem
when  $S_{\bx|y}$ degenerates. To overcome this problem, some other approaches should be used.
An interesting future research topic is to see whether overlapping technique can apply to
other slicing based dimension reduction methods such as sliced average variance estimation \cite{CookWeisberg1991}
and the sliced average third moment estimation \cite{yin2003estimating}
to improve the estimation accuracy as well as overcome the degeneracy phenomenon simultaneously.

\begin{appendices}

\section{Proof of Proposition \ref{Grelation}}
\label{proof:OSIR1}

We adopt the notations $s_0=s_{H+1} = \emptyset$, $\hat p_0=\hat p_{H+1} =0$,  and $\hat\bm_0=\hat\bm_{H+1}=\mathbf 0$.
This allows us to simplify the representations of the matrices of interest.

\begin{proof}
Without loss of generality, we can assume $\bar{\bx} = 0$. Then
$$\hat \Gamma_H = \sum_{h=1}^H \hat p_h \hat \bm_h \hat\bm_h^\top$$
and
$$\hat \Gamma_H^{(1)} = \sum_{h=0}^H \hat p_{h:(h+1)} \hat \bm_{h:(h+1)} \hat\bm_{h:(h+1)}^\top.$$
By $\hat p_{h:(h+1)} = \frac 12 (\hat p_h+\hat p_{h+1})$ and
$$\hat\bm_{h:(h+1)} = \frac{\hat p_h \hat\bm_h + \hat p_{h+1} \hat\bm_{h+1}}{\hat p_h +\hat p_{h+1}},$$
we have
$$\begin{array}{rcl}
& & 2 \hat p_{h:(h+1)} \hat\bm_{(h:(h+1)} \hat\bm_{(h:(h+1)} ^\top \\[1em]
& = & \dfrac 1 {\hat p_h +\hat p_{h+1}}
\Big( \hat p_h^2 \hat \bm_h \hat \bm_h^\top + p_{h+1}^2 \hat\bm_{h+1}\hat\bm_{h+1}^\top
+ \hat p_h\hat p_{h+1} \bm_h \hat\bm_{h+1}^\top   +  \hat p_h\hat p_{h+1} \bm_{h+1} \hat\bm_h^\top  \Big) \\[1em]
& = & \dfrac 1 {\hat p_h +\hat p_{h+1}}
\bigg\{ \hat p_h (\hat p_h + \hat p_{h+1}) \hat \bm_h \hat \bm_h^\top + p_{h+1} ((\hat p_h + \hat p_{h+1} ) \hat\bm_{h+1}\hat\bm_{h+1}^\top
\\[1em]
& & \qquad
- \hat p_h\hat p_{h+1} \Big( \hat \bm_h \hat \bm_h^\top -\bm_h \hat\bm_{h+1}^\top
-   \bm_{h+1} \hat\bm_h^\top +\hat\bm_{h+1}\hat\bm_{h+1}^\top  \Big)\bigg\} \\[1em]
& = &\Big( \hat p_h  \hat \bm_h \hat \bm_h^\top + p_{h+1} \hat\bm_{h+1}\hat\bm_{h+1}^\top \Big)
- \dfrac { \hat p_h\hat p_{h+1} }{\hat p_h +\hat p_{h+1}}  \Big(\hat\bm_{h+1} -\hat \bm_h \Big) \Big(\hat\bm_{h+1} -\hat \bm_h \Big)^\top.
\end{array}$$
Therefore,
\begin{align*}
2\hat \Gamma_H^{(1)} & = \dsum_{h=0}^H \Big( \hat p_h  \hat \bm_h \hat \bm_h^\top + p_{h+1} \hat\bm_{h+1}\hat\bm_{h+1}^\top \Big)  \\
& \qquad- \dsum_{h=0}^H \dfrac { \hat p_h\hat p_{h+1} }{\hat p_h +\hat p_{h+1}}  \Big(\hat\bm_{h+1} -\hat \bm_h \Big) \Big(\hat\bm_{h+1} -\hat
\bm_h \Big)^\top \\
& = 2 \dsum_{h=1}^ H  \hat p_h  \hat \bm_h \hat \bm_h^\top  -  \dsum_{h=1}^{H-1} \dfrac { \hat p_h\hat p_{h+1} }{\hat p_h +\hat p_{h+1}}
\Big(\hat\bm_{h+1} -\hat \bm_h \Big) \Big(\hat\bm_{h+1} -\hat \bm_h \Big)^\top \\
& = 2 \hat\Gamma _H -2 \hat D_{H}^{(1)}.
\end{align*}
This finishes the proof.
\end{proof}

\section{Proof of the $\sqrt{n}$ consistency}
\label{proof:consistency}

The following lemma was well known and a detailed proof can be found  in \cite{Yu2016distance}.
\begin{lemma} \label{Sigma}
Assume that $\bx$ has finite fourth moments. Let
$$S(\bx) = (\bx-\bmu)(\bx-\bmu)^\top -\Sigma.$$ Then
$$\hat\Sigma - \Sigma = \frac 1 n \sum_{i=1}^n S(\bx_i) + o_P\left(\frac 1 {\sqrt n}\right).$$
\end{lemma}

\begin{lemma} \label{Gamma1}
There exists a matrix-valued random variable $R(\bx, y)$ such that
$$\hat \Gamma_H^{(1)} - \Gamma_H^{(1)}  = \frac 1 n \dsum_{i=1}^n R(\bx_i, y_i) + o_P\left(\frac 1 {\sqrt n}\right).$$
\end{lemma}

\begin{proof}
Note that $$\hat p_{h:(h+1)} = \frac 1 {2n} \sum_{i=1}^n \b1_{h:(h+1)}(y_i)$$ and $p_{h:(h+1)} = \tfrac 1 2 \bE[\b1_{h:(h+1)}(y)].$ So
$$\hat p _{h:(h+1)} - p_{h:(h+1)}  = \frac 1 {2n} \sum_{i=1}^n \Big( \b1_{h:(h+1)}(y_i) - p_{h:(h+1)}\Big) = O_P\left(\frac 1
{\sqrt{n}}\right)$$
and
$$\frac 1 {\hat p _{h:(h+1)} }- \frac 1 { p_{h:(h+1)}}  = \frac 1 {2n p_{h:(h+1)} ^2}
\sum_{i=1}^n ( \b1_{h:(h+1)}(y_i) - p_{h:(h+1)}) + o_P\left(\frac 1 {\sqrt n}\right) = O_P\left(\frac 1 {\sqrt{n}}\right).$$
It is not difficult to check that
$p_{h:(h+1)} \bm_{h:(h+1} =  \bE[\bx \b1_{h:(h+1)}(y)]$  and
$$\hat p_{h:(h+1)} \hat \bm_{h:(h+1)} = \frac 1 n \sum_{i=1}^n \bx_i \b1_{h:(h+1)} (y_i).$$
So
$$\hat p_{h:(h+1)} \hat \bm_{h:(h+1)} - p_{h:(h+1)} \bm_{h:(h+1)}  =
\frac 1 n \sum_{i=1}^n \left(\bx_i \b1_{h:(h+1)} (y_i) - p_{h:(h+1)}\bm_{h:(h+1)}\right) = O_P\left(\frac 1 {\sqrt n}\right).$$
and
$$\begin{array}{rcl}
\hat \bm_{h:(h+1)} - \bm_{h:(h+1)} & = & \dfrac {\hat p_{h:(h+1)} \hat \bm_{h:(h+1)} }{\hat p_{h:(h+1)} } -
\dfrac { p_{h:(h+1)} \bm_{h:(h+1)} }{p_{h:(h+1)}} \\[1em]
& =  & \dfrac {1}{\hat p_{h:(h+1)}} \Big( \hat p_{h:(h+1)} \hat \bm_{h:(h+1)} - p_{h:(h+1)} \bm_{h:(h+1)}  \Big) \\[1em]
  & & \quad +  p_{h:(h+1)} \bm_{h:(h+1)}  \left( \dfrac 1 {\hat p _{h:(h+1)} }- \dfrac 1 { p_{h:(h+1)}}  \right) \\[1em]
  & = & \dfrac {1}{p_{h:(h+1)}}\Big( \hat p_{h:(h+1)} \hat \bm_{h:(h+1)} - p_{h:(h+1)} \bm_{h:(h+1)}  \Big) \\[1em]
  &  & \quad +  p_{h:(h+1)} \bm_{h:(h+1)}  \left( \dfrac 1 {\hat p _{h:(h+1)} }- \dfrac 1 { p_{h:(h+1)}}  \right)  + o_P\left(\dfrac 1 {\sqrt
  n}\right)\\[1em]
  & =  & \dfrac 1 n \dsum_{i=1}^n U_{h,1}(\bx_i, y_i) +  o_P\left(\frac 1 {\sqrt n}\right) \\[1em]
  & = & O_P\left(\dfrac 1 {\sqrt n}\right),
\end{array}$$
where $$U_1(\bx_i, y_i) = \frac {\bx_i \b1_{h:(h+1)} (y_i) }{p_{h:(h+1)}} - \bm_{h:(h+1)} +
\bm_{h:(h+1)} \left(\frac  {\b1_{h:(h+1)}(y_i)}{ p_{h:(h+1)} }- 1\right). $$
Therefore,
$$\begin{array}{rcl}
& & \hat p_{h:(h+1)} \hat \bm_{h:(h+1)}  \hat \bm_{h:(h+1)} ^\top - p_{h:(h+1)}  \bm_{h:(h+1)}   \bm_{h:(h+1)} ^\top \\[1em]
& = & \left(\hat p_{h:(h+1)} \hat \bm_{h:(h+1)}  - p_{h:(h+1)}  \bm_{h:(h+1)}  \right)  \hat \bm_{h:(h+1)} ^\top \\[1em]
& &\quad +  p_{h:(h+1)} \bm_{h:(h+1)} \left( \hat \bm_{h:(h+1)} ^\top -   \bm_{h:(h+1)} \right) ^\top \\[1em]
& = & \left(\hat p_{h:(h+1)} \hat \bm_{h:(h+1)}  - p_{h:(h+1)}  \bm_{h:(h+1)}  \right)  \bm_{h:(h+1)} ^\top \\[1em]
& & \quad+  p_{h:(h+1)} \bm_{h:(h+1)} \left( \hat \bm_{h:(h+1)}  -   \bm_{h:(h+1)} \right) ^\top  + o_P\left(\dfrac 1 {\sqrt n}\right)\\[1em]
  & = & \dfrac 1 n \dsum_{i=1}^n U(\bx_i, y_i) +  o_P\left(\dfrac 1 {\sqrt n}\right) \\[2em]
  & = & O_P\left(\dfrac 1 {\sqrt n}\right),
\end{array}$$
where $$U_h(\bx_i, y_i) =  \left(\bx_i \b1_{h:(h+1)} (y_i) - p_{h:(h+1)}\bm_{h:(h+1)}\right)\bm_{h:(h+1)} ^\top
+  p_{h:(h+1)} \bm_{h:(h+1)} U_{h,1}(\bx_i, y_i)^\top. $$

Note that $$\bar \bx -\bmu = \frac 1 n \sum_{i=1}^n (x_i-\mu) = O_P\left(\frac 1 {\sqrt n}\right).$$
We obtain
$$\begin{array}{rcl}
\bar \bx \bar\bx^\top - \bmu\bmu^\top
& = &(\bar\bx-\bmu) \bar\bx^\top + \bmu(\bar\bx-\bmu)^\top \\[1em]
& = & (\bar\bx-\bmu) \bmu^\top + \bmu(\bar\bx-\bmu)^\top + o_P(\tfrac 1 {\sqrt n}) \\[1em]
& = & \dfrac 1 n \dsum_{i=1}^n (\bx_i-\mu)\bmu^\top + \bmu(\bx_i-\mu)^\top + o_P(\tfrac 1 {\sqrt n}) .
\end{array}$$

By simple calculation we have
$$\hat \Gamma_H^{(1)} = \sum_{h=0}^H \hat p_{h:(h+1)} \hat \bm_{h:(h+1)}  \hat \bm_{h:(h+1)} ^\top - \bar\bx \bar\bx^\top$$
and
$$\Gamma_H^{(1)} = \sum_{h=0}^H  p_{h:(h+1)}  \bm_{h:(h+1)}  \bm_{h:(h+1)} ^\top - \bmu \bmu^\top$$
So
$$\begin{array}{rcl}
\hat \Gamma_H^{(1)} - \Gamma_H^{(1)} & = &\dsum_{h=0}^H
\left(  \hat p_{h:(h+1)} \hat \bm_{h:(h+1)}  \hat \bm_{h:(h+1)} ^\top - p_{h:(h+1)}  \bm_{h:(h+1)}   \bm_{h:(h+1)} ^\top\right) \\[1em]
& & \quad + \left( \bar \bx \bar\bx^\top - \bmu\bmu^\top\right) \\[1em]
& = & \frac 1 n \dsum_{i=1}^n R(\bx_i, y_i) + o_P(\frac 1 {\sqrt n})
\end{array}$$
with $$ R(\bx_i, y_i)  = \sum_{h=0}^H U_h(\bx_i, y_i) + (\bx_i-\mu)\bmu^\top + \bmu(\bx_i-\mu)^\top.$$
This finishes the proof.
 \end{proof}

\bigskip

 \begin{proof}[Proof of Theorem \ref{error}.]
 By perturbation theory and standard argument (see e.g. \cite{Yu2016distance}), we can obtain
$$\hat \lambda_k = \lambda_k + \beta_k^\top \left\{ (\hat \Gamma_H^{(1)}- \Gamma_H^{(1)}) + \lambda_k (\hat \Sigma -\Sigma) \right\} \beta_k$$
and
$$\hat\beta_k = \beta_k - \frac { \beta_k \beta_k^\top (\hat\Sigma-\Sigma)\beta_k}{2} - \sum_{j\not=k}
\frac{\beta_j \beta_j^\top  \left\{ (\hat \Gamma_H^{(1)}- \Gamma_H^{(1)}) + \lambda_K (\hat \Sigma -\Sigma) \right\} \beta_k}{\lambda_j
-\lambda_k}.$$
By using Lemma \ref{Sigma} and Lemma \ref{Gamma1} we obtain the desired estimation with
$$\xi_k(\bx, y) = \beta_k^\top \left\{ U (\bx, y) + \lambda_k S(\bx, y) \right\} \beta_k$$
and
$$\Upsilon_k(\bx, y) = - \frac { \beta_k \beta_k^\top S(\bx, y)\beta_k}{2} - \sum_{j\not=k}
\frac{\beta_j \beta_j^\top  \left\{ U(\bx, y) + \lambda_KS(\bx, y) \right\} \beta_k}{\lambda_j -\lambda_k}.$$
 \end{proof}

\section{Proof of Proposition \ref{Grelation2}}
\label{proof:OSIR2}

We again adopt the null slice notations 
$s_{-1} = s_0=s_{H+1} = s_{H+2}= \emptyset$, 
$\hat p_{-1} = \hat p_0=\hat p_{H+1} = \hat p_{H+2} =0$,  and $\hat\bm_{-1}
= \hat\bm_0=\hat\bm_{H+1}=\hat\bm_{H+2}=\mathbf 0$
to simplify the representations.

\begin{proof}
Without loss of generality, we can assume $\bar{\bx} = 0$. Then
$$\hat \Gamma_H = \sum_{h=1}^H \hat p_h \hat \bm_h \hat\bm_h^\top$$
and
$$\hat \Gamma_H^{(2)} = \sum_{h=-1}^H \hat p_{h:(h+2)} \hat \bm_{h:(h+2)} \hat\bm_{h:(h+2)}^\top.$$
By $\hat p_{h:(h+2)} = \frac 1 3 (\hat p_h+\hat p_{h+1}+\hat p_{h+2})$ and
$$\hat\bm_{h:(h+1)} = \frac{\hat p_h \hat\bm_h + \hat p_{h+1} \hat\bm_{h+1} + \hat p_{h+2}\hat\bm_{h+2}}{\hat p_h +\hat p_{h+1}+\hat
p_{h+2}},$$
we have
$$\begin{array}{rcl}
& & 3 \hat p_{h:(h+2)} \hat\bm_{(h:(h+2)} \hat\bm_{(h:(h+2)} ^\top \\[1em]
& = & \dfrac 1 {\hat p_h +\hat p_{h+1} +\hat p_{h+2}}
\Big( \hat p_h^2 \hat \bm_h \hat \bm_h^\top + p_{h+1}^2 \hat\bm_{h+1}\hat\bm_{h+1}^\top  + p_{h+2}^2 \hat\bm_{h+2}\hat\bm_{h+2}^\top \\[1em]
& & \qquad\qquad\qquad\qquad \quad
  + \, \hat p_h\hat p_{h+1} \bm_h \hat\bm_{h+1}^\top  +  \hat p_h\hat p_{h+1} \bm_{h+1} \hat\bm_h^\top\\[1em]
& & \qquad\qquad\qquad\qquad\quad
+ \, \hat p_{h+1}\hat p_{h+2} \bm_{h+1} \hat\bm_{h+2}^\top
+ \hat p_{h+1}\hat p_{h+2} \bm_{h+2} \hat\bm_{h+1}^\top \\[1em]
& & \qquad\qquad\qquad\qquad \quad
 + \, \hat p_h\hat p_{h+2} \bm_h \hat\bm_{h+2}^\top   +   \hat p_h\hat p_{h+2} \bm_{h+2} \hat\bm_h^\top \Big) \\[1em]
& = & \dfrac 1 {\hat p_h +\hat p_{h+1}+\hat p_{h+2}}
\bigg\{ \hat p_h (\hat p_h + \hat p_{h+1}+ \hat p_{h+2}) \hat \bm_h \hat \bm_h^\top + p_{h+1} ( \hat p_h + \hat p_{h+1}
  +\hat p_{h+2} )\hat\bm_{h+1}\hat\bm_{h+1}^\top \\[1em]
& & \qquad\qquad\qquad\qquad\quad
+p_{h+2} ((\hat p_h + \hat p_{h+1}+\hat p_{h+2} ) \hat\bm_{h+2}\hat\bm_{h+2}^\top \\[1em]
& & \qquad\qquad \qquad\qquad\quad
-\, \hat p_h\hat p_{h+1} \Big( \hat \bm_h \hat \bm_h^\top -\bm_h \hat\bm_{h+1}^\top
-  \bm_{h+1} \hat\bm_h^\top +\hat\bm_{h+1}\hat\bm_{h+1}^\top\Big) \\[1em]
& & \qquad\qquad \qquad\qquad\quad
- \, \hat p_{h+1}\hat p_{h+2} \Big( \hat \bm_{h+1} \hat \bm_{h+1}^\top -\bm_{h+1} \hat\bm_{h+2}^\top
-  \bm_{h+2} \hat\bm_{h+1}^\top +\hat\bm_{h+2}\hat\bm_{h+2}^\top \Big)\\[1em]
& & \qquad\qquad \qquad\qquad\quad
- \, \hat p_h\hat p_{h+2} \Big( \hat \bm_h \hat \bm_h^\top -\bm_h \hat\bm_{h+2}^\top
-  \bm_{h+2} \hat\bm_h^\top +\hat\bm_{h+2}\hat\bm_{h+2}^\top  \Big)\bigg\} \\[1em]
& = & \Big( \hat p_h  \hat \bm_h \hat \bm_h^\top + p_{h+1} \hat\bm_{h+1}\hat\bm_{h+1}^\top
+ p_{h+2} \hat\bm_{h+2}\hat\bm_{h+2}^\top \Big) \\[1em]
& & \qquad
- \dfrac { \hat p_h\hat p_{h+1} }{\hat p_h +\hat p_{h+1}+\hat\bm_{h+2}}
\Big(\hat\bm_{h+1} -\hat \bm_h \Big) \Big(\hat\bm_{h+1} -\hat \bm_h \Big)^\top \\[2em]
& & \qquad - \dfrac { \hat p_{h+1}\hat p_{h+2} }{+\hat p_{h}+\hat p_{h+1} +\hat p_{h+2}}
\Big(\hat\bm_{h+2} -\hat \bm_{h+1} \Big) \Big(\hat\bm_{h+2} -\hat \bm_{h+1} \Big)^\top\\[2em]
& & \qquad - \dfrac { \hat p_{h}\hat p_{h+2} }{+\hat p_{h}+\hat p_{h+1} +\hat p_{h+2}}
\Big(\hat\bm_{h+2} -\hat \bm_{h} \Big) \Big(\hat\bm_{h+2} -\hat \bm_{h} \Big)^\top.
\end{array}$$
By the fact that
$$\begin{array}{rcl}
 && \Big(\hat\bm_{h+2} -\hat \bm_{h} \Big) \Big(\hat\bm_{h+2} -\hat \bm_{h} \Big)^\top \\[1em]
 &= &  \Big( (\hat\bm_{h+2} -\hat\bm_{h+1})+ (\hat\bm_{h+1} -\hat \bm_{h}) \Big)
   \Big((\hat\bm_{h+2} -\hat\bm_{h+1})+ (\hat\bm_{h+1} -\hat \bm_{h}) \Big)^\top \\[1em]
 & = & 2\Big(\hat\bm_{h+2} -\hat \bm_{h+1} \Big) \Big(\hat\bm_{h+2} -\hat \bm_{h+1} \Big)^\top
 +2\Big(\hat\bm_{h+1} -\hat \bm_{h} \Big) \Big(\hat\bm_{h+1} -\hat \bm_{h} \Big)^\top \\[1em]
 & & \qquad - \Big(\hat\bm_{h+2} -2\hat \bm_{h+1} + \hat \bm_{h} \Big) \Big(\hat\bm_{h+2} -2\hat \bm_{h+1} + \hat \bm_{h}  \Big)^\top
\end{array} $$
Therefore,
\begin{align*}
3\hat \Gamma_H^{(1)} & = \dsum_{h=-1}^H \Big( \hat p_h  \hat \bm_h \hat \bm_h^\top + p_{h+1} \hat\bm_{h+1}\hat\bm_{h+1}^\top + p_{h+2}
\hat\bm_{h+2}\hat\bm_{h+2}^\top \Big)  \\
& \qquad- \dsum_{h=-1}^H \dfrac { \hat p_h\hat p_{h+1} +2\hat p_h\hat p_{h+2}}{\hat p_h +\hat p_{h+1} + \hat p_{h+2}}  \Big(\hat\bm_{h+1} -\hat
\bm_h \Big) \Big(\hat\bm_{h+1} -\hat \bm_h \Big)^\top \\
& \qquad- \dsum_{h=-1}^H \dfrac { \hat p_{h+1}\hat p_{h+2} +2\hat p_h\hat p_{h+2}}{\hat p_h +\hat p_{h+1} + \hat p_{h+2}}  \Big(\hat\bm_{h+2}
-\hat \bm_{h+1} \Big) \Big(\hat\bm_{h+2} -\hat \bm_{h+1} \Big)^\top \\
& \qquad+ \dsum_{h=-1}^H \dfrac {\hat p_h\hat p_{h+2}}{\hat p_h +\hat p_{h+1} + \hat p_{h+2}}  \Big(\hat\bm_{h+2} -2\hat \bm_{h+1}+\hat\bm_{h}
\Big) \Big(\hat\bm_{h+2} -2\hat \bm_{h+1}+\hat\bm_{h} \Big)^\top \\[0.5em]
& = 3 \hat\Gamma _H - \tilde D_{H}^{(1)} + \tilde D_{H}^{(2)}.
\end{align*}
This finishes the proof.
\end{proof}

\end{appendices}

\bibliographystyle{abbrv}
\bibliography{osirrefs}

\end{document}